\documentclass[letterpaper, 10 pt, journal, twoside]{IEEEtran}

\IEEEoverridecommandlockouts                              

\usepackage{multicol}
\usepackage[bookmarks=true]{hyperref}

\usepackage{float}
\usepackage{comment}
\usepackage{array}
\usepackage{graphicx}
\usepackage{booktabs}
\usepackage{multirow}
\usepackage[ruled]{algorithm2e}
\usepackage{hhline}
\usepackage[table]{xcolor}
\usepackage{dsfont}
\usepackage{hyperref}

\usepackage{amsmath}								
\usepackage{amssymb}
\usepackage{latexsym}
\usepackage{amsthm}
\usepackage{bm}
\usepackage{commath}
\usepackage{float}
\usepackage{units}

\usepackage{tikz}
\usetikzlibrary{arrows,shapes,trees,calc}
\usetikzlibrary{backgrounds}
\usepackage{pgfplots}
\pgfplotsset{compat=1.14}
\usepgfplotslibrary{polar}
\usepgfplotslibrary{patchplots}

\newcolumntype{C}[1]{>{\centering\arraybackslash}p{#1}}

\newtheorem{definition}{Definition}[section]
\newtheorem{theorem}{Theorem}[section]

\newtheorem{corollary}{Corollary}[section]

\definecolor{lightyellow}{RGB}{255,236,132}
\definecolor{lightgreen}{RGB}{161,239,10}
\definecolor{darkgreen}{RGB}{61,124,68}
\definecolor{lightblue}{RGB}{72,131,219}
\definecolor{darkblue}{RGB}{39,63,186}
\definecolor{plgreen}{RGB}{27,158,119}
\definecolor{plorange}{RGB}{217,95,2}
\definecolor{plpurple}{RGB}{117,112,179}
\definecolor{plpink}{RGB}{231,41,138}

\usepackage[compact]{titlesec}

\usepackage{color}  

\let\vec\bm    

\newcommand{\koop}{\mathcal{K}}

\newcommand{\Real}{\mathbb{R}}

\newcommand{\xcon}{\tilde{\vec{x}}}   
\newcommand{\ucon}{\tilde{\vec{u}}}   
\newcommand{\uconi}{\tilde{u}}   
\newcommand{\xconi}{\tilde{x}}   
\newcommand{\xvec}{\vec{x}}   
\newcommand{\uvec}{\vec{u}}   
\newcommand{\pvec}{\vec{p}}   
\newcommand{\qvec}{\vec{q}}   
\newcommand{\psivec}{\vec{\psi}}   
    
\newcommand{\Fvec}{\vec{F}}

\newcommand{\snap}[1]{^{(#1)}}  
\newcommand{\zarg}{(t)}   

\definecolor{lightyellow}{RGB}{255,236,132}
\definecolor{lightgreen}{RGB}{161,239,10}
\definecolor{darkgreen}{RGB}{61,124,68}
\definecolor{lightblue}{RGB}{72,131,219}
\definecolor{darkblue}{RGB}{39,63,186}
\definecolor{plgreen}{RGB}{27,158,119}
\definecolor{plorange}{RGB}{217,95,2}
\definecolor{plpurple}{RGB}{117,112,179}
\definecolor{plpink}{RGB}{231,41,138}

\usepackage{marginnote}     

\title{\LARGE \bf
Advantages of Bilinear Koopman Realizations for the\\Modeling and Control of Systems with Unknown Dynamics
}

\author{Daniel Bruder$^{1}$, 
        Xun Fu$^{1}$, %
        Ram Vasudevan$^{1}$%
\thanks{$^{1}$ The authors are with the Mechanical Engineering Department at the 
        University of Michigan, Ann Arbor, MI 48109, USA
        {\tt\small \{bruderd, xunfu, ramv\}@umich.edu}}%
}

\begin{document}
\setlength{\textfloatsep}{8pt}
\linespread{0.86}    
\setlength{\parskip}{0cm}   

\maketitle
\thispagestyle{empty}         
\pagestyle{empty}             


\begin{abstract}
%
%
Nonlinear dynamical systems can be made easier to control by lifting them into the space of observable functions, where their evolution is described by the linear Koopman operator.
This paper describes how the Koopman operator can be used to generate approximate linear, bilinear, and nonlinear model realizations from data, and argues in favor of bilinear realizations for characterizing systems with unknown dynamics.
Necessary and sufficient conditions for a dynamical system to have a valid linear or bilinear realization over a given set of observable functions are presented and used to show that every control-affine system admits an infinite-dimensional bilinear realization, but does not necessarily admit a linear one.
Therefore, approximate bilinear realizations constructed from generic sets of basis functions tend to improve as the number of basis functions increases, whereas approximate linear realizations may not.
To demonstrate the advantages of bilinear Koopman realizations for control, 
a linear, bilinear, and nonlinear Koopman model realization of a simulated robot arm are constructed from data. In a trajectory following task, the bilinear realization exceeds the prediction accuracy of the linear realization and the computational efficiency of the nonlinear realization when incorporated into a model predictive control framework.
\end{abstract}


\IEEEpeerreviewmaketitle


\section{Introduction}  \label{sec:introduction}

Linear systems theory provides a wealth of analysis and control design techniques for linear dynamical systems,
but no such general framework exists for all nonlinear dynamical systems.
For this reason, it is common to approximate a nonlinear system with one or more linear systems to make it compatible with well-established and computationally efficient linear control methods.

Linearization of nonlinear systems can be achieved in several ways.
Local linearization techniques, such as Jacobian linearization, describe the local behavior near equilibrium points using the first order Taylor series of the dynamics \cite{astrom2010feedback}.
Such linearizations preserve the stability properties of equilibrium points, but only predict system behavior well in their vicinity.

Global linearization techniques, on the other hand, aim to convert a nonlinear dynamical system into an equivalent linear system.
One such technique, based on Koopman operator theory, achieves this by lifting the system into a higher-dimensional space of scalar-valued functions called \emph{observables}.
In this space, the evolution of observables along trajectories of the nonlinear system are described by the (linear) Koopman operator \cite{budivsic2012applied}.
The Koopman operator captures the behavior of the system everywhere, not just near equilibrium points, providing a global linear representation of the system in terms of observables.

 Despite their linearity in the space of observables, Koopman representations have limitations that hinder their ability to inform the control of arbitrary nonlinear dynamical systems.  
One limitation is that linearity with respect to observables does not imply linearity with respect to the control input.
Therefore, Koopman models are not necessarily compatible with computationally efficient linear control design techniques such as LQR \cite{anderson2007optimal} and linear MPC \cite{rawlings2009model}.
A further limitation of Koopman linear embeddings is that they are in general infinite-dimensional.
Therefore, finite-dimensional truncations must be used in practice, which merely approximate the behavior of the original nonlinear system rather than capture it perfectly.

Linearity with respect to the control input can be enforced by restricting the Koopman operator to a subspace of observables which only includes linear functions of the control input.
This strategy is introduced in \cite{korda2018linear} to construct linear predictors of nonlinear systems for model predictive control.
This and similar approaches have shown capable of achieving better control performance than local linearization techniques on a variety of different robots such as a Sphero SDK \cite{abraham2017model}, a swimming fish robot \cite{mamakoukas2019local}, a Rethink Sawyer robot arm \cite{abraham2019active}, and a several soft robots \cite{Vasudevan-RSS-19, bruder2020koopman}.

Such applications make use of finite-dimensional matrix approximations of the Koopman operator that are identified from data using a linear system identification technique known as Extended Dynamic Mode Decomposition (EDMD) \cite{williams2015data, mauroy2016linear, mauroy2017koopman}. 
These approximations rely upon the implicit assumption that the identified Koopman matrix converges to the true Koopman operator as its dimension goes to infinity.
In other words, they assume that sufficient accuracy can always be achieved by a Koopman model with linear control inputs if its dimension is large enough.
However, as shown in \cite{bakker2019koopman}, a valid Koopman representation is not guaranteed to exist when the approximation is restricted to a subspace that excludes nonlinear functions of the input.
This holds true even if the subspace is infinite-dimensional.

There is thus a trade-off when looking to utilize Koopman representations for control of a system with unknown dynamics.
Koopman representations with linear control inputs are advantageous because they are compatible with computationally efficient linear control techniques.
However, a valid Koopman representation of this form may not exist, in which case finite-dimensional approximations of this type are unlikely to improve with dimension.
On the other hand, every system admits a valid Koopman representation with nonlinear control inputs.
However, such representations have limited usefulness for control applications because they are not compatible with efficient control techniques.

Koopman representations with bilinear control input terms strike a compromise between these two extremes.
Such models preserve some of the computational benefits of linear Koopman models, while being more likely to exist for arbitrary dynamical systems.
This makes them desirable for control applications.

Previous work has investigated Koopman-based bilinearization and control methods for nonlinear systems.
In \cite{goswami2020global}, the Koopman Canonical Transform (KCT) is used to describe how to construct bilinear realizations over the set of Koopman eigenfunctions.
It also introduces a set of sufficient conditions for a nonlinear system to be bilinearizable.
In \cite{peitz2020data}, it is shown that the control-affine property of dynamical systems is preserved for the continuous Koopman operator, which allows bilinear surrogate models to be constructed by interpolating between different operators.
These bilinear models are incorporated into a model predictive control framework and applied to the control of several nonlinear systems.

In line with previous work,
this paper explores the benefits of using the Koopman operator to build bilinear models of systems with unknown dynamics.
While approximate Koopman model realizations with linear, bilinear, or nonlinear control inputs can be directly constructed from data, 
we argue that bilinear Koopman realizations identified in this manner are likely to yield better predictions than linear ones, and be more computationally efficient than nonlinear ones.

The contributions of this paper are twofold.
First, we offer a theoretical justification for the proposed advantages of bilinear Koopman model realizations.
We specify necessary and sufficient conditions for a dynamical system to have a valid linear or bilinear realization over a given set of observables, and use these conditions to show show that every control-affine system admits a bilinear realization, but does not necessarily admit a linear one.
Second, we provide instructions on how to construct approximate linear, bilinear, and nonlinear Koopman model realizations from data.
Using this approach, we construct several Koopman model realizations of a simulated robot arm, and show that a bilinear realization simultaneously exceeds the prediction accuracy of a linear realization and the computational efficiency of a nonlinear realization when incorporated into a model predictive control framework.

The paper is organized as follows.
In Section \ref{sec:theory}, we present some preliminaries of Koopman operator theory as well as the necessary and sufficient conditions for linear and bilinear Koopman model realizations to exist.
In Section \ref{sec:methods}, we describe the process for identifying approximate Koopman model realizations from data as well as how they can be incorporated into a model predictive control framework.
In Section \ref{sec:experiments}, we present modeling and control comparisons for several linear, bilinear, and nonlinear Koopman model realizations of a simulated planar robot arm.
In Section \ref{sec:discussion}, we discuss the results of these experiments and offer some concluding remarks.

\section{Theory}    \label{sec:theory}

This section describes the Koopman operator and how it relates to model realizations.
A definition for linear and bilinear realizations is presented, as well as necessary and sufficient conditions for a realization to be of one of these types.
Unpacking these conditions reveals that every control-affine system has a (possibly infinite-dimensional) bilinear realization, but not necessarily a linear one.

\subsection{Koopman operator preliminaries}  \label{sec:preliminaries}

Consider the (nonlinear) dynamical system governed by the following differential equation
\begin{align}
    \dot{\xvec}(t) &= \Fvec ( \xvec(t) , \uvec(t) )
    \label{eq:nlsys}
\end{align}
where $\xvec(t) = [x_1(t) , \ldots , x_n(t) ]^\top \in X \subset \Real^n$ is the state and $\uvec(t) = [u_1(t) , \ldots , u_m(t) ]^\top \in U \subset \Real^m$ is the input of the system at time $t \in [0 , +\infty)$, $\Fvec$ is a continuously differentiable function, and $X, U$ are compact subsets.
Denote by $\phi_\tau : X \times U \to X \times U$ the \emph{flow map}, where $\phi_\tau ( \xvec(0) , \uvec(0) ) = ( \xvec(\tau) , \uvec(0) )$ and $\xvec(\tau)$ is the solution to \eqref{eq:nlsys} at time $\tau$ when beginning with the initial condition $\xvec(0)$ at time $0$ and a constant input $\uvec(0)$ applied for all time between $0$ and $\tau$.

Let $\mathcal{F}$ be the infinite-dimensional function space composed of all square-integrable real-valued functions with compact domain $X \times U \subset \Real^{n \times m}$.
Elements of $\mathcal{F}$ are called \emph{observables}.
In $\mathcal{F}$, the flow of the system is characterized by the set of Koopman operators $\koop_\tau : \mathcal{F} \to \mathcal{F}$, for each $\tau \geq 0$,
which describe the evolution of every observable ${f \in \mathcal{F}}$ along the trajectories of the system according to the following definition:
\begin{align}
    \koop_\tau f = f \circ \phi_\tau,      
    \label{eq:koopman}
\end{align}
where $\circ$ indicates function composition such that
\begin{align}
    ( \koop_\tau  f) (\xvec(t),\ucon ) 
    &= f( \xvec(t+\tau), \ucon ) .
    \label{eq:koopman-applied}
\end{align}
for a constant input $\ucon$ over the time interval $[t,t+\tau]$.

The set of Koopman operators is generated by an infinitesimal generator, denoted $\koop$, according to the following relation \cite[Chapter 11]{higham2008functions},
\begin{align}
    \koop_\tau &= e^{\tau \koop}
    \label{eq:cont-2-dis}
\end{align}
The infinitesimal generator describes the time derivative of observables along trajectories of the system
\begin{align}
    \koop f &= \frac{d}{dt} f
    \label{eq:generator}
\end{align}
such that
\begin{align}
    ( \koop f ) (\xvec(t) , \uvec(t)) &= \frac{\partial f}{\partial \xvec} \Fvec (\xvec(t) , \uvec(t) )
    \label{eq:generator-applied}
\end{align}
For a given time-step, $\koop_\tau$ is sometimes referred to as the \emph{discrete time} Koopman operator and $\koop$ is referred to as the \emph{continuous time} Koopman operator \cite{brunton2016koopman}.

\subsection{Model Realizations} \label{sec:realizations}

A model realization of \eqref{eq:nlsys} is a dynamical system that generates the same state response $\xvec$ as \eqref{eq:nlsys} under any input signal $\uvec$.
The Koopman operator can be used to generate model realizations.
This is done by choosing a countable set of observables $\{ f_i \in \mathcal{F} \}_{i=1}^N$ (where $N \in \mathbb{N} \cup \infty$) from which the original state can be recovered through an inverse mapping $C : \Real^N \to \Real^n$.
Taken together, this set of observables constitutes the \emph{lifted state} of the realization. 
The evolution of each component of the lifted state is described by applying the continuous Koopman operator, and the output equation of the realization consists of the mapping $C$:
\begin{align}
    \frac{d}{dt} f_i (t) &= \koop f_i (t)
    \hspace{30pt} \text{for } i = 1,\ldots,N \\
    \xvec(t) &= C \left( f_1 (t) , \ldots , f_N (t) \right)
    \label{eq:realization}
\end{align}
where $f_i (t)$ is shorthand for $f_i (\xvec(t),\uvec(t))$.
We say that the realization is defined over the set of observables $\{ f_i \}_{i=1}^N$.
Note that model realizations are not unique since they depend on the choice of observables.

The Koopman operator is a linear operator, therefore a realization generated using the Koopman operator is linear with respect to the observables over which it is defined.
However, because these observables can be nonlinear functions, such realizations are not necessarily linear with respect to the original state and input.
Thus, they may not present any advantages in terms of computational efficiency.

Certain types of realizations are particularly amenable to control design since the control input appears either linearly or bilinearly in them. 
Before formally defining these realizations,
we first define the following subsets of the space of all observables $\mathcal{F}$ for notational convenience:
\begin{align}
    \mathcal{X} &:= \left\{ f \in \mathcal{F} \,\vert\, f(\xcon,\ucon) = \tilde{x}_i \text{ for some } i = 1,...,n \right\} \\
    \mathcal{U} &:= \left\{ f \in \mathcal{F} \,\vert\, f(\xcon,\ucon) = \tilde{u}_i \text{ for some } i = 1,...,m \right\} \\
    \mathcal{Z} &:= \left\{ f \in \mathcal{F} \,\vert\, f(\xcon,\ucon_1) = f(\xcon,\ucon_2)\,,\, \forall \ucon_1,\ucon_2 \in \Real^m \right\}
\end{align}
where $\xcon = [\tilde{x}_1 , \ldots , \tilde{x}_n] \in X$ and $\ucon = [\tilde{u}_1 , \ldots , \tilde{u}_m]\in U$.
In other words, $\mathcal{X}$ is the set of functions that project onto components of the state, $\mathcal{U}$ is the set of functions that project onto components of the input, and $\mathcal{Z}$ is the set of all functions that depend on the state only, including constant functions.
With this notation in hand, we define linear and bilinear model realizations as follows:

\begin{definition}[Linear and bilinear model realizations]
A model realization of $\dot{\xvec}(t) = \Fvec ( \xvec(t) , \uvec(t) )$ over the set of observables ${ \{ z_i \in \mathcal{Z} \}_{i=1}^N }$ is \textbf{bilinear} if there exist sets of coefficients 
${ \{ a_{ij} \in \Real \}_{i=1,j=1}^{N,N} }$, 
${ \{ b_{ij} \in \Real \}_{i=1,j=1}^{N,m} }$, 
${ \{ h_{ijk} \in \Real \}_{i=1,j=1,k=1}^{N,N,m} }$, 
and ${ \{ c_{ij} \in \Real \}_{i=1,j=1}^{n,N} }$
such that
\begin{equation}
\begin{aligned}
    \frac{d}{dt} z_i \zarg 
    = &\sum_{j=1}^N a_{ij} z_j \zarg
    \hspace*{-1pt} + \hspace*{-1pt} \sum_{j=1}^m b_{ij} u_j (t)
    \hspace*{-1pt} + \hspace*{-1pt} \sum_{j=1}^m \sum_{k=1}^N h_{ijk} z_{k} \zarg u_{j}(t)   
    \label{eq:realization-dynamics}
\end{aligned}
\end{equation}
for $i = 1 , \ldots , N$ and
\begin{align}
    x_i(t) &= \sum_{j=1}^N c_{ij} z_j \zarg
    \label{eq:realization-output}
\end{align}
for $i = 1 , \ldots , n$ where $z_i \zarg$ is shorthand for $z_i (\xvec(t),\uvec(t))$, $\xvec(t) = [x_1(t) , \ldots , x_n(t) ]^\top$, and $\uvec(t) = [u_1(t) , \ldots , u_m(t) ]^\top $.
If ${ h_{ijk} = 0 }$ for all $i,j,k$, then the realization is said to be \textbf{linear}.
\label{def:realizations}
\end{definition}
\noindent The lifted state of these realizations consists entirely of observables that depend on the state only, and the input appears only in linear and bilinear terms.


As stated previously, realizations are not unique.
They depend on the choice of observables over which they are defined.
For a particular choice of observables to yield a bilinear or linear realization they must satisfy certain conditions.
These conditions are presented in the following theorem.

\begin{theorem}[Necessary and sufficient conditions for linear and bilinear realizations] \label{thm:realizations}

The realization of the system governed by \eqref{eq:nlsys} over a set of observables ${ \bar{\mathcal{Z}} = \{ z_i \in \mathcal{Z} \}_{i=1}^N }$ is:
\begin{enumerate}
    \item \textbf{Bilinear} if and only if
    \begin{align}
        \frac{\partial z_i}{\partial \xvec} \Fvec \in \text{span}\left( \bar{\mathcal{Z}} \cup \mathcal{U} \cup \{ f \cdot g \vert f \in \bar{\mathcal{Z}} , g \in \mathcal{U} \} \right)
        \label{eq:cond-bilinear}
    \end{align}
    for $i=1,\ldots,N$
    and $\mathcal{X} \subset \text{span}(\bar{\mathcal{Z}})$.
    \item \textbf{Linear} if an only if
    \begin{align}
        \frac{\partial z_i}{\partial \xvec} \Fvec \in \text{span}\left( \bar{\mathcal{Z}} \cup \mathcal{U} \right)
        \label{eq:cond-linear}
    \end{align}
    for $i=1,\ldots,N$
    and $\mathcal{X} \subset \text{span}(\bar{\mathcal{Z}})$.
\end{enumerate}
Note that linear systems are by definition also bilinear.
\end{theorem}

\begin{proof}
The definition for both realizations specifies that the state can be recovered as a linear combination of the observables, i.e.
\begin{align}
    x_i &= \sum_{j=1}^N c_{ij} z_j,
    &&\text{for } i = 1,\ldots,n \\
    \Leftrightarrow
    x_i &\in \text{span}\left( \bar{\mathcal{Z}} \right), &&\text{for } i = 1,\ldots,n \\
    \Leftrightarrow
    \mathcal{X} &\subset \text{span}( \bar{\mathcal{Z}} ).
\end{align}
%
The conditions specified by \eqref{eq:cond-bilinear} and \eqref{eq:cond-linear} are verified by applying the chain rule to the left hand side of \eqref{eq:realization-dynamics}:
\begin{align}
    \frac{d}{dt} z_i &= \frac{\partial z_i}{\partial \xvec} \frac{d \xvec}{dt} = \frac{\partial z_i}{\partial \xvec} \Fvec,
    &&&&\text{for } i = 1,\ldots,N
    \label{eq:chain-rule}
\end{align}
Plugging \eqref{eq:chain-rule} into the definition of a bilinear realization from Def.~\ref{def:realizations} then yields,
\begin{align}
    \frac{\partial z_i}{\partial \xvec} \Fvec 
    &= \sum_{j=1}^N a_{ij} z_j + \sum_{j=1}^m b_{ij} u_j 
    + \sum_{j=1}^m \sum_{k=1}^N h_{ijk} z_{k} u_{j} \\
    \Leftrightarrow 
    \frac{\partial z_i}{\partial \xvec} \Fvec 
    &\in \text{span}\left( \bar{\mathcal{Z}} \cup \mathcal{U} \cup \{ f \cdot g \vert f \in \bar{\mathcal{Z}} , g \in \mathcal{U} \} \right)
\end{align}
Similarly, plugging \eqref{eq:chain-rule} into the definition of a linear realization from Def.~\ref{def:realizations} yields,
\begin{align}
    \frac{\partial z_i}{\partial \xvec} \Fvec 
    &= \sum_{j=1}^N a_{ij} z_j + \sum_{j=1}^m b_{ij} u_j \\
    \Leftrightarrow
    \frac{\partial z_i}{\partial \xvec} \Fvec &\in \text{span}( \bar{\mathcal{Z}} \cup \mathcal{U} )
\end{align}

\end{proof}



Since there are computational benefits to performing control synthesis with linear and bilinear realizations, we would prefer to choose sets of observables that admit a realization of one of these types.
However, unless the dynamics of the system are known, it is not possible to verify whether a particular set of observables satisfies the conditions outlined in Theorem \ref{thm:realizations}.
Nevertheless, as we show below, one can prove that for control-affine systems, a realization over $\mathcal{Z}$ is bilinear.



\begin{corollary}
If the system governed by \eqref{eq:nlsys} is control-affine and a set of observables $\bar{\mathcal{Z}} = \{ z_i \in \mathcal{Z} \}_{i=1}^\infty$ is a basis of $\mathcal{Z}$, then the realization of the system defined over $\bar{\mathcal{Z}}$ is bilinear.
\end{corollary}
\begin{proof}
The system being control-affine implies that there exists $\Fvec_x : X \to \Real^n$ and $\{ \Fvec_u : X \to \Real^n \}_{i=1}^m$ such that
\begin{align}
    \Fvec( \xvec , \uvec ) &= \Fvec_x (\xvec) + \sum_{j=1}^m \Fvec_u^j (\xvec) u_j
\end{align}
Consider a realization generated by the Koopman operateor as in \eqref{eq:realization}, then the function $\frac{\partial z_i}{\partial \xvec}$ maps from $X$ to $\Real^{1 \times n}$.
Thus, the products of $\frac{\partial z_i}{\partial \xvec}$ with $\Fvec_x$ and $\Fvec_u^j$ map from $X$ to $\Real$, i.e. they are in $\mathcal{Z}$.
Given that $\bar{\mathcal{Z}}$ is a basis for $\mathcal{Z}$, it follows that
\begin{align}
    \frac{\partial z_i}{\partial \xvec} \Fvec_x \in \text{span}\{ \bar{\mathcal{Z}} \}, &&\forall i \\
    \frac{\partial z_i}{\partial \xvec} \Fvec_u^j \in \text{span}\{ \bar{\mathcal{Z}} \}, &&\forall i,j
\end{align}
This implies that for some set of coefficients $\{ \gamma_{ij} \}_{i=1,j=1}^{\infty,m}$,
\begin{align}
    \frac{\partial z_i}{\partial \xvec} \Fvec_u^j &= \sum_{i=1}^\infty \gamma_{ij} z_i, &&\forall j \\
    \implies \frac{\partial z_i}{\partial \xvec} \Fvec_u^j u_j &= \sum_{i=1}^\infty \gamma_{ij} z_i u_j, &&\forall j \\
    \implies \frac{\partial z_i}{\partial \xvec} \Fvec_u^j u_j &\in \text{span}\left( 
    \{ f \cdot g \vert f \in \bar{\mathcal{Z}} , g \in \mathcal{U} \}
    \right), &&\forall j
\end{align}
Thus, \eqref{eq:cond-bilinear} is satisfied.
Additionally, the condition that ${ \mathcal{X} \subset \text{span}(\bar{\mathcal{Z}}) }$ is satisfied since ${ \mathcal{X} \subset \mathcal{Z} }$ and ${ \mathcal{Z} = \text{span}(\bar{\mathcal{Z}}) }$.
\end{proof}

A consequence of the preceding corollary is that every control-affine system has a valid bilinear realization over an infinite set of basis functions.
Thus, generic sets of basis functions (i.e. polynomial, Fourier, or radial) can be used as the chosen set of observables in $\bar{\mathcal{Z}}$ since linear combinations of them can represent arbitrary functions in $\mathcal{Z}$.
It should be noted that the theorem offers no such guarantee that a system
has a valid linear realization, no matter the choice of basis functions.
Therefore, realizations defined over a finite set of basis functions of the space of state-dependent observables
will converge to a valid bilinear realization as the number of basis functions goes to infinity, but will not necessarily converge to a valid linear realization.


\section{Techniques for Data-Driven Realization}   \label{sec:methods}

The previous section describes how a model realization can be constructed over a set of observables using the Koopman operator.
This section describes how to identify matrix approximations of the Koopman operator from data such that they can be used to produce model realizations that are linear, bilinear, or nonlinear.

\subsection{Approximation of the Koopman operator from data}
\label{sec:koopid}

The restriction of the Koopman operator to a finite-dimensional subspace can be represented as a matrix.
Using the Extended Dynamic Mode Decomposition (EDMD) algorithm \cite{williams2015data, mauroy2016linear, mauroy2017koopman}, we identify a finite-dimensional matrix approximation of the Koopman operator via linear regression applied to observed data.

We first specify a finite-dimensional subspace as the span of a chosen set of $M$ linearly independent observables ${ \{ \psi_i : X \times U \to \Real \}_{i=1}^M}$.
We then define a \emph{lifting function} ${\psivec : X \times U \to \Real^M}$ which evaluates each of the observables and stacks them into a vector:
\begin{align}
    \psi( \xcon , \ucon ) &:= \begin{bmatrix} \psi_1( \xcon , \ucon ) & \cdots & \psi_M ( \xcon , \ucon ) \end{bmatrix}^\top,
    \label{eq:lift}
\end{align}
where $\xcon \in X$ and $\ucon \in U$.

To approximate the discrete time Koopman operator from a set of experimental data, we take $K$ discrete measurements in the form of so-called ``snapshots'' $\{ \pvec\snap{k} , \qvec\snap{k}, \uvec\snap{k} \}_{k=1}^{K}$ where
\begin{align}
    \pvec\snap{k} &:=  \xvec ( t\snap{k} ) \\
    \qvec\snap{k} &:= \xvec ( t\snap{k} + T_s ),
    \label{eq:ab}
\end{align}
$t\snap{k}$ denotes the time corresponding to the $k^\text{th}$ measurement, $\uvec\snap{k}$ is a constant input applied between $\pvec\snap{k}$ and $\qvec\snap{k}$, and $T_s$ is the sampling period, which is assumed to be identical for all snapshots.
Note that consecutive snapshots do not have to be generated by consecutive measurements. 

By definition, the action of the discrete time Koopman operator ${\koop}_{T_s}$ advances the value of observables one time-step.
Therefore, the best matrix approximation of it in the least-squares sense, denoted $\bar{\koop}_{T_s}$, is the minimizer to
\begin{align}
    \underset{\check{\koop}}{\min} 
    \sum_{k=1}^K
    \left\lVert 
    {\check{\koop}}^\top \psivec( \pvec\snap{k} , \uvec\snap{k} ) - \psivec( \qvec\snap{k} , \uvec\snap{k} )
    \right\rVert_2^2.
    \label{eq:Kbar}
\end{align}

With an approximation of the discrete time Koopman operator $\bar{\koop}_{T_s}$ in hand, we can solve for the corresponding continuous time Koopman operator $\bar{\koop}$ by inverting \eqref{eq:cont-2-dis}:
\begin{align}
    \bar\koop &= \frac{1}{T_s} \log{ \koop_{T_s} } 
    \label{eq:dis-2-cont}
\end{align}
where $\log$ denotes the principal matrix logarithm \cite[Chapter 11]{higham2008functions}.

%
%
\subsection{Linear Model Realization}
\label{sec:linid}


The Koopman matrix $\bar{\koop}$ on the subspace spanned by a set of observables $\{ \psi_i : X \times U \to \Real \}_{i=1}^{N+m}$ can be constructed such that it is decomposable into a linear system representation.
One way to achieve this is to define the first $N$ basis functions as functions of the state only, and the last $m$ basis functions as projections onto each component of the input, i.e.
\begin{align}
    \psi_i := 
    \begin{cases}
        z_i &\in \mathcal{Z}, \quad \text{for } i = 1,\ldots, N  \\
        \pi_{u_{i-N}} &\in \mathcal{U}, \quad \text{for } i = N+1, \ldots, N+m
    \end{cases}
\end{align}
where $\pi_{u_i}$ denotes the projection onto the $i^{\text{th}}$ component of the input.
This choice ensures that the input only appears in the last $m$ components of the lifted vector $\psivec(\xcon,\ucon)$.

The Koopman matrix can be identified from data according the steps laid out in Section~\ref{sec:koopid}.
Then, by construction, the coefficients for a linear realization of the form specified in \eqref{eq:realization-dynamics} are embedded within the first $N$ rows of the transpose of the Koopman matrix in the following manner,
\begin{align}
    \bar{\koop}^\top &= 
    \begin{bmatrix} 
        A_{\{N \times N\}} & B_{\{N \times m\}}
        \\
        \vdots & \vdots
    \end{bmatrix}
    \label{eq:AB}
\end{align}
where
\begin{align}
    &A = \begin{bmatrix}
            a_{11} & \cdots & a_{1N} \\
            \vdots & \ddots & \vdots \\
            a_{N1} & \cdots & a_{NN}
        \end{bmatrix},
    &&B = \begin{bmatrix}
            b_{11} & \cdots & b_{1m} \\
            \vdots & \ddots & \vdots \\
            b_{N1} & \cdots & b_{Nm}
        \end{bmatrix}
        \label{eq:A_i-B_i}
\end{align}
and the subscripts in curly brackets denote the dimensions of each matrix.

For convenience, we can define the first $n$ basis functions as projections onto each component of the state, i.e.
\begin{align}
    \psi_i &:= \pi_{x_i} \in \mathcal{X} \subset \mathcal{Z} \quad \text{for } i = 1, \ldots, n  \label{eq:psi-1-n}
\end{align}
Then, the coefficients of the output equations of \eqref{eq:realization-output} can be defined as simply
\begin{align}
    c_{ij} &= 
    \begin{cases}
        1, &\text{for } i=j \\
        0, &\text{for } i\neq j
    \end{cases}
    \label{eq:C}
\end{align}


%
%
\subsection{Bilinear Model Realization}
\label{sec:bilinid}


With a suitable choice of observables $\{ \psi_i \}_{i=1}^{N(m+1)+m}$, the Koopman matrix $\bar{\koop}$ can be constructed such that it is decomposable into a bilinear system representation.
The first $N$ observables are defined as functions of the state only,
the next $Nm$ observables are defined as the product of the first $N$ basis functions and each component of the input, and the last $m$ basis functions are defined as projections onto each component of the input, i.e.
\begin{align}
    \psi_i &:=
    \begin{cases}
        z_i \in \mathcal{Z}, &\text{for } i = 1,\ldots, N  \\
        z_{i-N} \cdot \pi_{u_{1}}, &\text{for } i = N+1, \ldots, 2N \\
        \hspace{20pt} \vdots & \hspace{25pt} \vdots \\
        z_{i-Nm} \cdot \pi_{u_{m}}, &\text{for } i = Nm+1, \ldots, N(m+1) \\
        \pi_{u_{i-N(m+1)}} \in \mathcal{U}, &\text{for } i = N(m+1)+1, \ldots,\\
        &\hspace{32pt} N(m+1)+m
    \end{cases}
    \label{eq:basis-bilinear}
\end{align}

The Koopman matrix can be identified from data according the steps laid out in Section~\ref{sec:koopid}.
Then, by construction, the coefficients for a bilinear realization of the form specified in \eqref{eq:realization-dynamics} are embedded within the first $N$ rows of the transpose of the Koopmn matrix in the following manner,
\begin{align}
    \bar{\koop}^\top &= 
    \begin{bmatrix} 
        A &
        H_{1} &
        \cdots & 
        H_{m} &
        B &
        \\
        \vdots & \vdots & \vdots & \vdots & \vdots
    \end{bmatrix}
    \label{eq:ABH}
\end{align}
where $A$ and $B$ are defined as in \eqref{eq:A_i-B_i} and
\begin{align}
    H_i &=
    \begin{bmatrix}
        h_{i11} & \cdots & h_{i1N} \\
        \vdots & \ddots & \vdots \\
        h_{iN1} & \cdots & h_{iNN}
    \end{bmatrix}.
    \label{eq:H_i}
\end{align}

Just as with the linear model realization, we can define the first $n$ basis functions as in \eqref{eq:psi-1-n}, and define the coefficients of the output equation by \eqref{eq:C}


%
%
\subsection{Nonlinear Model Realization}
\label{sec:nlid-discrete}

If the Koopman matrix $\bar{\koop}$ is identified on a subspace spanned by a set observables ${ \{ \psi_i : X \times U \to \Real \}_{i=1}^M}$ other than the type specified in Sections~\ref{sec:linid} and \ref{sec:bilinid}, then the realization it produces will be nonlinear, i.e. it may contain nonlinear functions of the input.
Such a representation takes the form of a linear combination of the basis functions for each component of the state:
\begin{align}
    x_i &= \sum_{j=1}^M c_{ij} \psi_j, && \text{for } i = 1,\ldots,n
    \label{eq:nonlinear-system}
\end{align}

Assuming once again that the first $n$ observables are defined as projections onto each component of the state as in \eqref{eq:psi-1-n},
by construction, the coefficients are embedded within the first $n$ rows of the transpose of the Koopman matrix, i.e.
\begin{align}
    \bar{\koop}^\top &=
    \begin{bmatrix}
        c_{11} & \cdots & c_{1M} \\
        \vdots & \ddots & \vdots \\
        c_{n1} & \cdots & c_{nM} \\
        \vdots & \vdots & \vdots
    \end{bmatrix}
\end{align}


%
%
\subsection{Model Predictive Control}
\label{sec:control}

Model predictive control algorithms optimally choose a sequence of control inputs given a desired output trajectory and system model \cite{rawlings2009model}.
This system model can be either linear or nonlinear, but linear models have computational advantages over nonlinear ones.
Namely, the MPC optimization problem for linear models is convex, while for nonlinear models it is not.


The MPC optimization problem for linear systems is convex because it has a quadratic cost function and linear constraints.
Since it is convex, it has a unique globally optimal solution that can efficiently be constructed without initialization even for models with thousands of states and inputs \cite{boyd2004convex, paulson2014fast,stellato2018osqp}.
This contrasts sharply with the MPC formulation for nonlinear systems
which requires solving an optimization problem with nonlinear constraints and a (potentially) nonlinear cost function \cite{allgower2012nonlinear}.
As a result, algorithms to solve such problems typically require initialization and can struggle to find globally optimal solutions \cite{polak2012optimization}.
Though techniques have been proposed to improve the speed of algorithms to solve nonlinear MPC problems \cite{patterson2014gpops,hereid2017frost} or even globally solve such problems without requiring initialization \cite{zhao2017control}, these formulations still take orders of magnitude more time per iteration, which can make them too slow to be applied for real-time control.

In this paper we implement three model predictive controller algorithms which we refer to by the following abbreviations:
\begin{itemize}
    \item K-MPC: Koopman-based linear MPC
    \item K-BMPC: Koopman-based bilinear MPC
    \item K-NMPC: Koopman-based nonlinear MPC
\end{itemize}
The K-MPC and K-NMPC controllers are standard model predictive controllers that use linear and nonlinear Koopman model realizations to generate predictions, respectively.
Similarly, the K-BMPC controller uses a bilinear Koopman model realization to generate predictions.
However, bilinear constraints would render the problem non-convex, so we instead rely on a linear approximation constructed by fixing the value of the lifted state in the bilinear terms over the prediction horizon, $N_h$.
The dynamics constraints in the K-BMPC optimization problem then become:
\begin{align}
    z_i[t+1] &= \sum_{j=1}^N a_{ij} z_j[t] + \sum_{j=1}^m b_{ij} u_j[t] + \sum_{j=1}^m \sum_{k=1}^N h_{ijk} z_{k}[0] u_{j}[t] \\
    &= \sum_{j=1}^N a_{ij} z_j[t] + \sum_{j=1}^m \left( b_{ij} + \sum_{k=1}^N h_{ijk} z_{k}[0] \right) u_j[t] 
    \label{eq:bilinear-linearized}
\end{align}
for $i = 1, \ldots , N$ and $t = 1,\ldots,N_h$, 
where $[\cdot]$ denotes a discrete time index.
The resulting linear dynamics approximate the behavior of the bilinear realization in a neighborhood of the initial lifted state $\{ z_i [0] \}_{i=1}^N$.
This introduces prediction error, but turns it into a convex quadratic program which can be solved quickly enough to be updated at every time-step.

The main difference between the K-BMPC controller and the other model predictive control algorithms is that it does not generate optimal solutions with respect to the model upon which it is based.
Instead, the solutions it generates are optimal with respect to a linear approximation of the actual bilinear model predictive control problem.
It is important to note, however, that all finite-dimensional Koopman realizations constructed from data are mere approximations of the true dynamics of a system.
Hence, even the optimal solutions to the model predictive control problems are not necessarily optimal with respect to the true dynamics of the system.

\section{Experiments and Results}
\label{sec:experiments}

To evaluate the relative performance of linear, bilinear, and nonlinear Koopman realizations for modeling and control of a realistic robotic system, we applied the methods from Section~\ref{sec:methods} to the simulated planar arm system shown in Fig.~\ref{fig:3-link-arm}.

The arm has 3-links each of mass $100$g and length $0.33$m and 3 joints each with a stiffness of $1\times 10^{-5}$ N/rad and a viscous damping coefficient of $1$ Ns/rad.
The input into the system is a set of $m=3$ applied joint torques and the output is the location of the end of each link expressed as a $n=6$ dimensional vector of Cartesian coordinates,
\begin{align}
    \vec{u}(t) &= [ \tau_1(t) ,\, \tau_2(t) ,\, \tau_3(t) ]^\top \\
    \vec{x}(t) &= [ \alpha_1(t) ,\, \beta_1(t) ,\, \alpha_2(t) ,\, \beta_2(t) ,\, \alpha_3(t) ,\, \beta_3(t) ]^\top
\end{align}
as shown in Fig.~\ref{fig:3-link-arm}.
A set of 12000 snapshots was collected for a time step of $T_s = 0.05$ seconds over a range of randomized initial conditions and inputs for the purposes of identifying matrix approximations of the Koopman operator as well as linear, bilinear, and nonlinear model realizations of the system.

\begin{figure}
    \centering
    \includegraphics[width=0.5\linewidth]{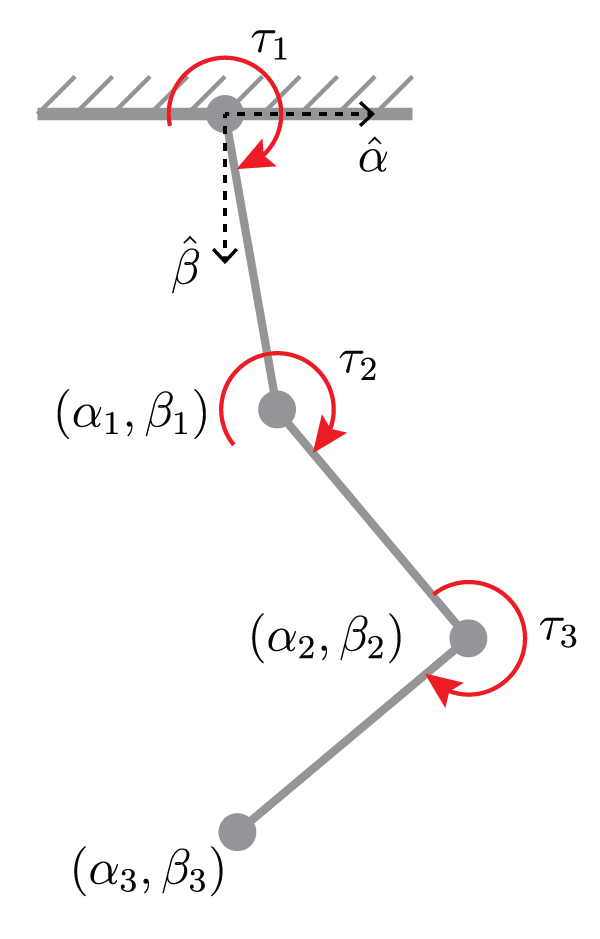}
    \caption{Three link planar arm system with input defined as joint torques and output defined as the locations of the end of each link.}
    \label{fig:3-link-arm}
\end{figure}

\subsection{Model Prediction Comparison}
\label{sec:sim-model-comparison}

The predictive accuracy of various models identified using the Koopman approach depends on both the model type (e.g. linear, bilinear, or nonlinear) as well as the number of basis functions.
We identified several models of each type on subspaces spanned by monomial basis functions.

The sets of basis functions used for identifying the linear models consisted of all monomials of the components of the state up to a specific degree denoted $\rho$, plus the projections onto each component of the input, i.e.
\begin{equation}
\begin{aligned}
    \{ \psi_i( \xcon , \ucon ) \}_{i=1}^{M} =& 
    \left\{
        \xconi_1^{\rho_1} \cdots \xconi_n^{\rho_n} 
        \middle|
        \rho_1 + \cdots + \rho_n \leq \rho
    \right\} \\
    &\hspace{6pt} \cup
    \left\{ 
        \uconi_i \middle| i \in \{1,...,m\} 
    \right\}
    \label{eq:lin-basis}
\end{aligned}
\end{equation}
where $M = N+m$ and $N = (n+\rho)!/(n!\rho!)$.
The sets of basis functions used for identifying bilinear models consisted of the same monomials as well as the product of each monomial with each component of the input, i.e.
\begin{equation}
\begin{aligned}
    \{ \psi_i( \xcon , \ucon ) \}_{i=1}^{M} = 
    \left\{ 
        ( \xconi_1^{\rho_1} \cdots \xconi_n^{\rho_n} ) \hat{u}
        \middle|
        \rho_1 + \cdots + \rho_n \leq \rho , \right. \\ \left.
        \hat{u} \in \{ 1 , \uconi_1 , ... , \uconi_m \}
    \right\}
    \label{eq:bilin-basis}
\end{aligned}
\end{equation}
where $M = (N+1)(m+!)$ and $N = (n+\rho)!/(n!\rho!)$.
The sets of basis functions used for identifying nonlinear models consisted of monomials up to degree $\rho$ of both the input and the state, i.e.
\begin{equation}
\begin{aligned}
    \{ \psi_i( \xcon , \ucon ) \}_{i=1}^M = 
    \left\{ 
        ( \xconi_1^{\rho_1} \cdots \xconi_n^{\rho_n} ) (\uconi_1^{\rho_{n+1}} \cdots \uconi_m^{\rho_{n+m}} )
        \middle| \right. \\ \left.
        \rho_1 + \cdots + \rho_{n+m} \leq \rho
    \right\}
    \label{eq:nonlin-basis}
\end{aligned}
\end{equation}
where $M = (n+m+\rho)!/((n+m)!\rho!)$.

Six linear models were identified for values of $\rho = 1,..,6$,
three bilinear models were identified for values of $\rho = 1,2,3$,
and four nonlinear models were identified for values of $\rho=1,...,4$.
Table \ref{tab:rho-vs-dimension} indicates the total number of basis functions, i.e. dim$\left(\psivec(\xcon,\ucon)\right)$, used for identifying the Koopman operator matrix for each model.
The prediction accuracy of each model was evaluated by comparing model simulations to validation data and computing the average error over all validation data points.
This error was quantified as the Euclidean distance between the predicted and real outputs in $\Real^6$.

Fig.~\ref{fig:error-vs-dimension} displays the prediction error verses the dimension of the Koopman operator matrix for each model.
The error shown in the plot is normalized by the average error incurred by the zero response.
The accuracy of the linear model increases very little, even when the dimension of the system is increased by several orders of magnitude.
The bilinear and nonlinear models become more accurate as the dimension increases.

\begin{table}
    \setlength\tabcolsep{10pt} 
    \centering
    \caption{Number of Monomial Basis Functions for Identified Models}
    \begin{tabular}{|C{0.75cm}|C{1cm}|C{1cm}|C{1cm}|}
        \hline
        & \multicolumn{3}{c |}{\textbf{$\#$ of basis functions, i.e. $M$}} \\
        \hhline{~---}
        \multirow{-2}{*}{$\bm{\rho}$} & Linear & Bilinear & Nonlinear \\
        \hline
        1 & 10 & 28 & 10 \\
        2 & 31 & 112 & 55 \\
        3 & 87 & 336 & 220 \\
        4 & 213 & - & 715 \\
        5 & 465 & - & - \\
        6 & 927 & - & - \\
        \hline
    \end{tabular}
    \label{tab:rho-vs-dimension}
\end{table}

\begin{figure}
    \centering
    \includegraphics[width=\linewidth]{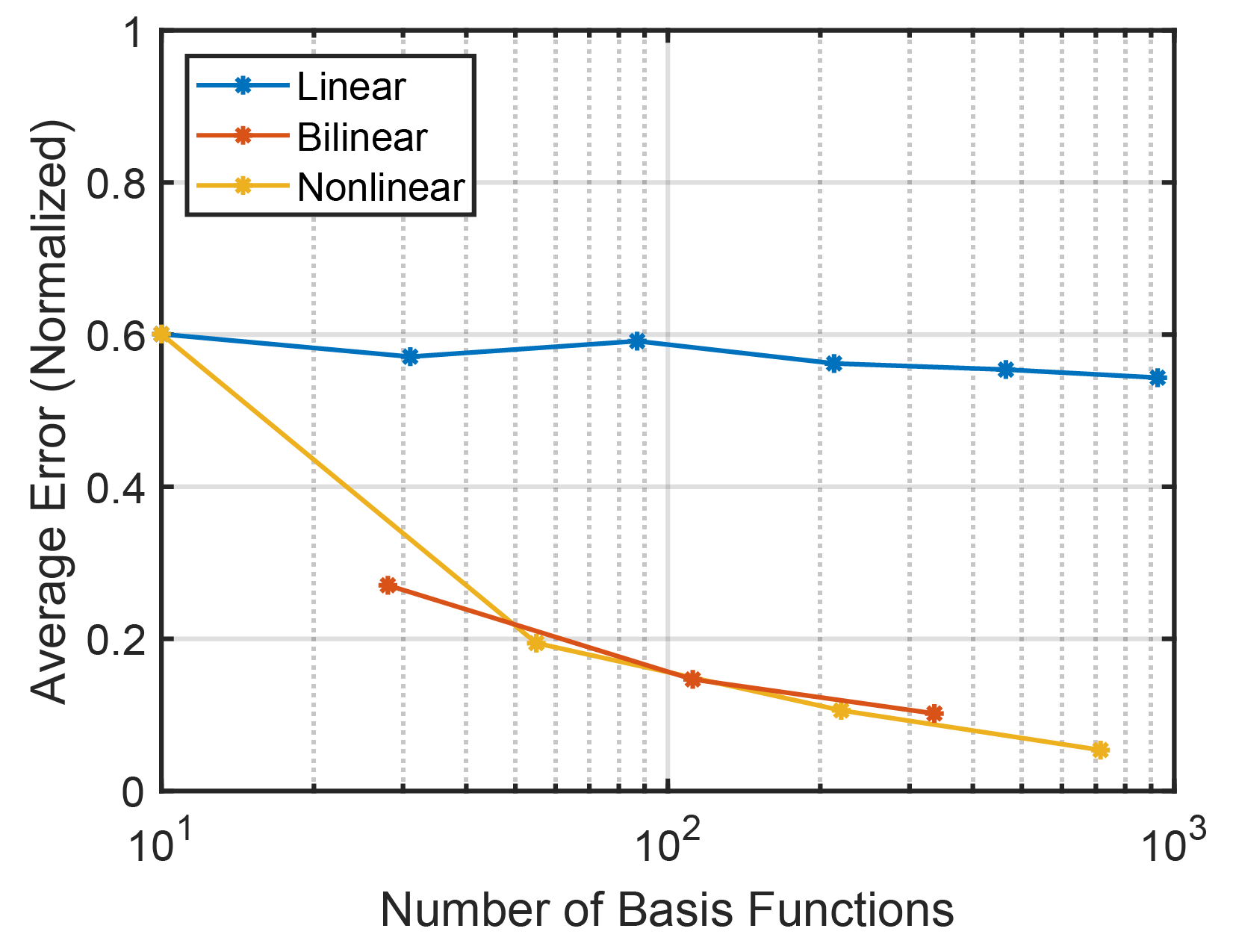}
    \caption{The model prediction error for several linear, bilinear, and nonlinear Koopman model realizations identified from the same set of data. As the number of basis functions increases, the error of the linear model changes little and the the error of the bilinear and nonlinear models decrease monotonically.}
    \label{fig:error-vs-dimension}
\end{figure}

%
%
\subsection{Control Performance Comparison}
\label{sec:sim-control-comparison}

To highlight the relative strengths and weaknesses of the Koopmans-based control techniques described in Section~\ref{sec:sim-model-comparison}, we applied them to the simulated 3-link planar arm system.
A K-MPC controller, K-BMPC controller, and K-NMPC controller was constructed from the linear, bilinear, and nonlinear model realizations identified in Section~\ref{sec:sim-model-comparison} for $\rho = 3$, respectively.
Each controller was then employed to perform the same trajectory following task.
Video of this experiment can be found in a supplementary video file\footnote{\href{https://youtu.be/9uLqNpnfr7M}{https://youtu.be/9uLqNpnfr7M}}.

Each controller computed solutions over a $N_h = 10$ step horizon and had identical cost functions.
The desired task was to move the end effector of the arm along a planar reference trajectory.
Therefore, a cost function was chosen that penalizes the distance between the actual end effector coordinates $(\alpha_3,\beta_3)$ and the desired coordinates $(\alpha_3^\text{ref},\beta_3^\text{ref})$ at each time-step,
as well as the magnitude of the input to distinguish between multiple robot configurations that all achieve the desired end effector placement.

The control experiment was conducted in simulation.
At each time-step, the current output of the system is measured and used to initialize the MPC optimization problem.
Once a solution is computed, the system is simulated forward one time-step ($T_s = 0.05$ seconds) under the optimal input.
This procedure is repeated until the end of the reference trajectory is reached.

The reference trajectory traces out the shape of a block letter M over a time period of 15 seconds, starting from the robot's hanging position.
Fig.~\ref{fig:mpc-comparison-blockM} shows the path of the end effector using each of the controllers, and Fig.~\ref{fig:mpc-comparison} displays the mean tracking error and mean computation time over all time-steps.
The tracking error at each time-step is quantified as the Euclidean distance between the actual and desired end effector locations. The computation time per iteration is the amount of time it takes to solve the MPC optimization problem and does not include the time to simulate the response of the system.
All three trials were run on a computer with 64 GB RAM and a 2.4~GHz CPU. 

\begin{figure*}
    \centering
    \includegraphics[width=0.9\linewidth]{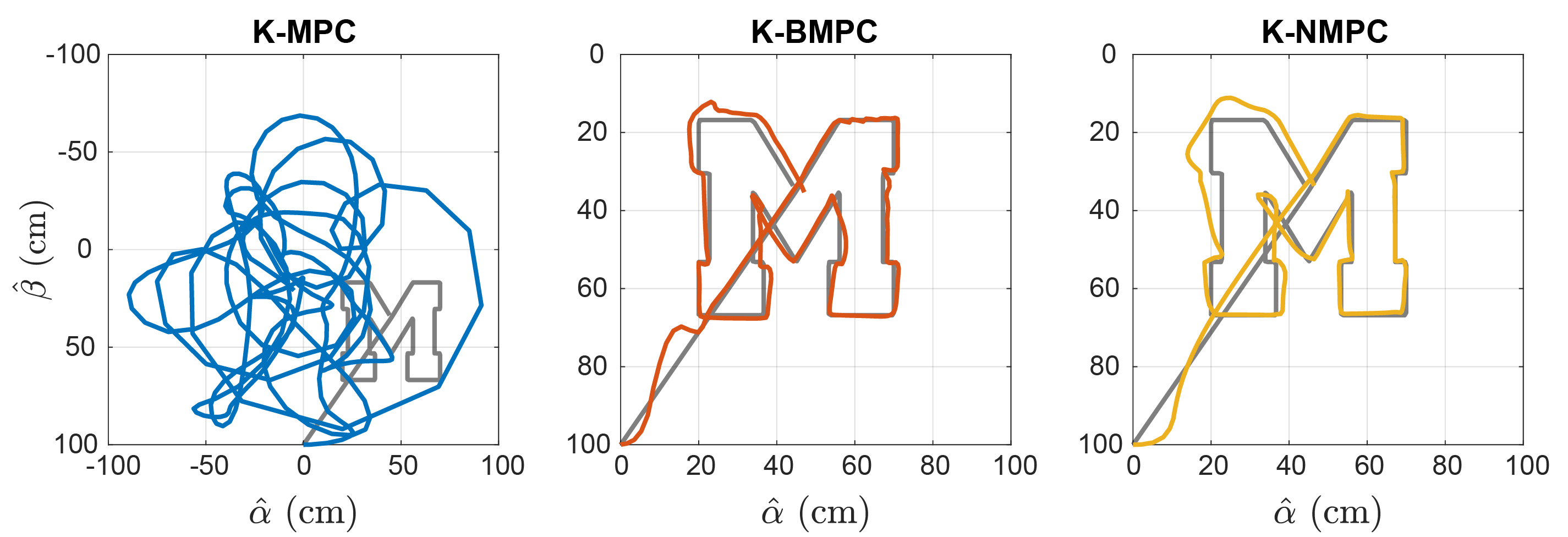}
    \caption{The end effector trajectories generated by each Koopman-based model predictive controller superimposed over the same reference trajectory, shown in grey.}
    \label{fig:mpc-comparison-blockM}
\end{figure*}

\begin{figure}
    \centering
    \includegraphics[width=0.9\linewidth]{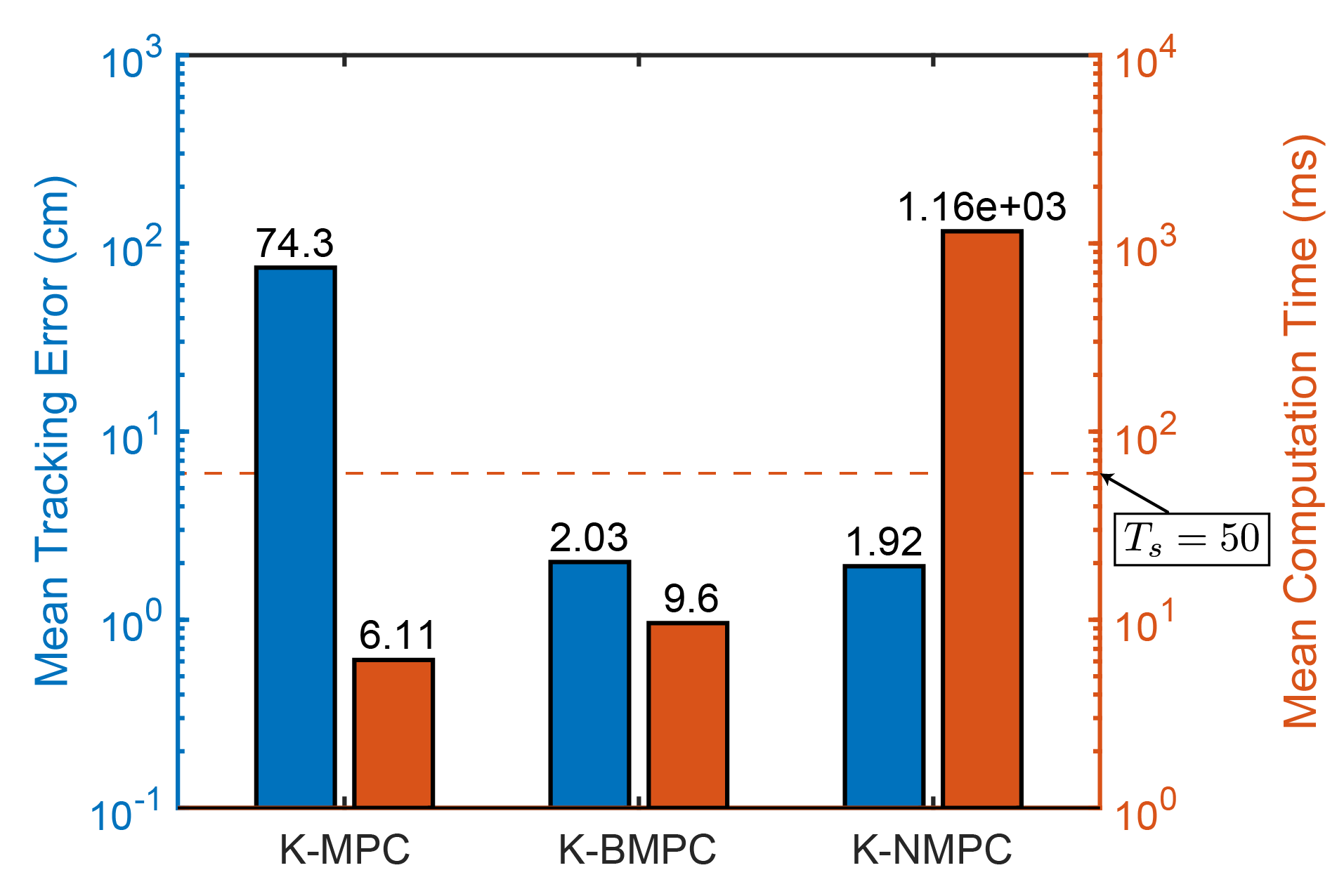}
    \caption{
    The mean tracking error (blue) and the mean computation time (orange) for each controller plotted on a logarithmic scale.
    The K-BMPC controller has a mean tracking error comparable to K-NMPC and a mean computation time comparable to K-MPC, proving it to be both accurate and computationally efficient.
    }
    \label{fig:mpc-comparison}
\end{figure}
\section{Discussion and Conclusion}    \label{sec:discussion}


The results of the model prediction comparison described in Section~\ref{sec:sim-model-comparison} illustrates the advantages of bilinear realizations for systems with unknown dynamics.
As expected, as the number of basis functions increases, the prediction error of the bilinear model decreases, indicating progress toward a true infinite-dimensional bilinear realization.
The linear model, on the other hand, does not improve with the inclusion of more basis functions, indicating that an infinite-dimensional linear realization over monomial basis functions probably doesn't exist.


It is clear by inspection of Fig.~\ref{fig:mpc-comparison-blockM} that the K-MPC controller performs very poorly.
This is confirmed quantitatively in Fig.~\ref{fig:mpc-comparison} which shows that its mean tracking error is more than 15$\times$ larger than that of the other controllers.
This poor performance can be attributed to the inaccuracy of the linear Koopman model realization upon which it is based, which is documented in Fig.~\ref{fig:error-vs-dimension}.
The K-BMPC and K-NMPC controllers track the desired trajectory with much greater fidelity, reflecting the greater accuracy of the bilinear and nonlinear model realizations upon which they are based.

In Section~\ref{sec:realizations}, we asserted that K-NMPC is much less computationally efficient than the other controllers, and that is confirmed by the results of this experiment.
As seen in Fig.~\ref{fig:mpc-comparison}, the mean computation time for K-NMPC is more than 500$\times$ larger than that of the other two controllers.
This computation time greatly exceeds the $50$~ms duration of a single time-step, making it incompatible with closed-loop operation.
Hence, if this robot were a real physical system, the control inputs would have to be computed offline.

Based on the results of this experiment, only K-BMPC would be a viable closed-loop controller for this system.
Its mean computation time is much less than a single time-step, and despite the suboptimality of its solutions, its mean tracking error is nearly equivalent to that of K-NMPC.
Roughly speaking, K-MPC fast but inaccurate, K-NMPC is accurate but slow, and K-BMPC is both fast and accurate.



This works shows that when the dynamics of a system are completely unknown, a bilinear Koopman realization constructed from data is likely to yield better overall modeling and control results than a linear or nonlinear Koopman realization. 
This is justified theoretically in Section~\ref{sec:theory} and demonstrated practically in Section~\ref{sec:experiments}.
We proved that control-affine systems have infinite-dimensional bilinear realizations but not necessarily linear ones.
Therefore, approximate bilinear realizations constructed from generic sets of basis functions improve as the number of basis functions increases, whereas approximate linear realizations may not.

Bilinear realizations combine the computational efficiency of linear realizations with the prediction accuracy of nonlinear realizations.
However, the bilinear model predictive control framework used in this paper could likely be improved.
Further work should investigate theoretical guarantees for bilinear controllers, and explore new approaches to optimal control of bilinear dynamical systems.





\bibliographystyle{IEEEtran}
\bibliography{references}

\end{document}